\documentclass{article}
\usepackage{ijcai17}

\usepackage{times}

\usepackage{latexsym}
\usepackage{graphicx}
\usepackage{color}
\usepackage{multirow}
\usepackage{amsmath}
\usepackage{amssymb}
\usepackage{psfig}
\usepackage{subcaption}
\usepackage{tabularx}

\usepackage[ruled,vlined,linesnumbered]{algorithm2e}
\newtheorem{definition}{Definition}
\newtheorem{lemma}{Lemma}
\newtheorem{corollary}{Corollary}

\newenvironment{proof}{\noindent{\bf Proof:~~}}{\qed}
\newcommand{\tuple}[1]{\ensuremath{\left \langle #1 \right \rangle }}

\newcommand{\comment}[1]{}

\newcommand{\qed}{\hfill\ensuremath{\blacksquare}}

\newcommand{\SAS}{SAS$^+$}

\newcommand{\cmfp}{safe model-free planning}
\newcommand{\eff}{\textit{eff}}
\newcommand{\pre}{\textit{pre}}
\newcommand{\solvable}{\textit{S}}
\newcommand{\plannable}{\textit{P}}

\newtheorem{theorem}{Theorem}

\title{Efficient, Safe, and Probably Approximately Complete Learning of Action Models}
\author{Roni Stern 
\\ Ben Gurion University of the Negev 
\\ Be'er Sheva, Israel
\\ roni.stern@gmail.com
\And Brendan Juba 
\\ Washington University in St. Louis
\\ St. Louis, MO, 63130 USA
\\ bjuba@wustl.edu
}

\begin{document}
	
	\maketitle
	
	\begin{abstract}
		In this paper we explore the theoretical boundaries of 
		planning in a setting where no model of the agent's actions is given. 
		Instead of an action model, a set of successfully executed plans are 
		given and the task is to generate a plan that is {\em safe}, i.e., guaranteed to achieve the goal without failing. 
		To this end, we show how to learn a conservative model of the world 
		in which actions are guaranteed to be applicable. 
		This conservative model is then given to an off-the-shelf classical planner, 
		resulting in a plan that is guaranteed to achieve the goal. 
		However, this reduction from a model-free planning to a 
		model-based planning is not complete: in some cases 
		a plan will not be found even when such exists. 
		We analyze the relation between the number of observed plans 
		and the likelihood that our conservative approach will indeed fail to solve a solvable problem. Our analysis show that the number of trajectories needed scales gracefully.
	\end{abstract}
	
	\section{Introduction}
	In classical planning problems, a model of the acting agent and its relationship to the relevant world is given in a formal planning description language, e.g., the classical STRIPS model~\cite{fikes1971strips} or PDDL~\cite{mcdermott1998pddl}. Planning algorithms (planners) use this model to generate a plan for achieving a given goal condition. 	
	Creating a planning domain model, however, is
	acknowledged as a notoriously hard knowledge engineering task.
	This has motivated much work on {\em learning} such knowledge.
	
	One such well-studied approach is to learn a domain model by observing the agent's interactions with the environment. The problems that arise in such approaches, however, are frequently intractable~\cite{kearns1994,daniely2016}. An alternative approach
	that is commonly used in {\em reinforcement learning} is to skip the model-learning phase and directly learn how to act by observing the agent's past actions and observations, and by guiding the agent towards performing exploratory actions~\cite[inter alia]{kearns2002}. In most prior work, the agent may fail to execute a planned action. When this occurs, the agent replans, possibly refining an underlying domain or action model it has learned. Thus, the agent learns from both positive and negative examples.

	
	In this work we address a different setting, in which such execution failures must be avoided. This setting is relevant when execution failures are very costly 
	or when the agent has limited computation capabilities, 
    and thus does not have the capability to re-plan after the plan it has tried to execute has failed.
	Consider, for example, a team of nano-bots deployed inside a human body for medical target identification and drug delivery~\cite{cavalcanti2007nanorobot}. Re-planning is likely not possible in such nano-bots, and, of course, failing to cure the disease is undesirable. Thus, the planning task we focus on in this paper is how to find a plan that is {\em safe}, i.e., is guaranteed to achieve the goal, in a setting where a domain model is not available. We call this problem the {\em safe model-free planning problem}. 
	

	Since performing actions that might fail is not allowed, exploration actions 
	cannot be performed. The only source of information available is a set of 
	{\em trajectories} of previously executed plans. 
	First, we show how to learn a set of actions that the agent can use from the given trajectories. 
	For every such action $a$, we bound the set of predicates that are $a$'s preconditions and effects. This bounded action model is then used  to construct a classical planning problem such that a solution to it is a solution to our model-free problem. This approach to solve the model-free planning problem 
	is sound and can be very efficient, since current classical planners are very efficient.   
	However, it is not complete, as the planning problem that uses the learned action model might not be solvable even if the underlying model-free planning problem is. 
	Nonetheless, we prove that under some assumptions, the probability of this occurring decreases quasi-linearly with the number of observed trajectories.

	This positive result comes in contrast to the hardness of other tasks related to model learning. 
	For example, learning to predict labels computed by finite-state machines~\cite{kearns1994} or even DNF formulas~\cite{daniely2016} is believed to be computationally intractable. Thus, we cannot hope to predict the values that fluents will take merely on the basis of the fact that these can be computed by a finite-state machine or a DNF. 
		Similarly, even the problems of finding memoryless policies~\cite{littman1994} or finite-state policies~\cite{meuleau1999} in simple environments is computationally intractable. Finally, in the standard interactive learning model, simple examples (that would be captured by a STRIPS environment, for example) are known to force a learner to explore an exponential number of paths in any reasonable measure of the environment's complexity~\cite[Section 8.6]{kakade2003}.

The key difference between the model learning we propose and these hardness results is that we limit our attention to the learning of STRIPS domain models in a PAC (``Probably Approximately Correct'') sense: we do not aim to learn an accurate action model, only one that is sufficient for finding a safe plan in most cases. We introduce and desribe this PAC-like setting in Section~\ref{learning-sec}.%

	\section{Problem Definition}\label{problemdef-sec}
	The setting we address is a a STRIPS planning domain, represented 
	using the \SAS{} formalism~\cite{backstrom1995complexity}. 
	A planning domain in SAS$+$ is defined by the tuple $\mathcal{D}=\tuple{\mathcal{X}, \mathcal{O}}$, where 
	\begin{itemize}
		\item $\mathcal{X}=\{x_1,\ldots, x_n\}$ is a set of state variables, each associated with a finite domain $Dom(x_i)$. 
		\item $\mathcal{O}$ is a set of actions, where each action $a\in \mathcal{O}$      is defined by a tuple $\tuple{\pre(a), \eff(a)}$, where $\pre(a)$ and $\eff(a)$ are assignments of values to state variables, i.e., a set of assignments of the form $x_i=v$ where $v\in Dom(x_i)$. We refer to $\mathcal{O}$ and its associated sets of preconditions and effects as the {\em action model} of the domain.
	\end{itemize}
	
	A state is also a set of assignments of the form $x_i=v$ 
	such that every variable in $\mathcal{X}$ is assigned a single value from its corresponding domain. As a shorthand notation, if a state $s$ contains an assignment $x_i=v$ we will write $s(x_i)=v$. A planning problem in \SAS{} is defined by the tuple $\Pi=\tuple{\mathcal{D},s_\mathcal{I}, s_\mathcal{G}}$, where 
	$s_\mathcal{I}$ is the initial state and $s_\mathcal{G}$ 
	is a partial assignment of some state variables that defines the goal. A state $s$ is a goal state if $s_\mathcal{G}\subseteq s$. 
	For an action $a$ and a state $s$ we denote by $apply(s,a)$ the state
	resulting from applying $a$ on state $s$. 
	A solution to an \SAS{} planning problem is a {\em plan}, 
	which is a sequence of actions $a_1,\ldots,a_m$ such that  $s_\mathcal{G}\subseteq apply(\ldots apply(s,a_1), a_2)\ldots a_m)$. 
	
	The key challenge we address in this paper is how to solve a \SAS{} 
	planning problem $\Pi=\tuple{\mathcal{D},s_\mathcal{I}, s_\mathcal{G}}$ without having the action model of $\mathcal{D}$. Instead, a set of {\em trajectories} are assumed to be given. 
	\begin{definition}[Trajectory]
		A trajectory $T=\langle s_1, a_1, s_2, a_2, \ldots, a_{n-1}, s_n\rangle$ is an alternating sequence of states ($s_1,\ldots,s_n$) and actions ($a_1,\ldots,a_n$) that starts and ends with a state. 
	\end{definition}
	A trajectory represents a successful execution of a sequence of actions by the agent. A set of trajectories may be obtained, for example, by monitoring the acting agent when it is controlled manually by a human operator. The states in the given trajectories are assumed to be fully observable. 
	
	Finally, we can define the {\em safe model-free planning} problem, which is the focus of this work. 
	
	\begin{definition}[Safe model-free planning]
		Let $\Pi=\tuple{\tuple{\mathcal{X},\mathcal{O}},s_\mathcal{I}, s_\mathcal{G}}$ be a planning 
		problem and let $\mathcal{T}=\{\mathcal{T}^1,\ldots,\mathcal{T}^m\}$ be a set of trajectories in 
		the planning domain $\tuple{\mathcal{X},\mathcal{O}}$. 
		The input to a safe model-free planning problem is 
		the tuple $\tuple{\mathcal{X},s_\mathcal{I}, s_\mathcal{G}, \mathcal{T}}$
		and the task is to generate a plan $\pi$  that is a solution to $\Pi$. 
		We denote this safe model-free planning problem as $\Pi_\mathcal{T}$. 
		\label{def:model-free-planning}
	\end{definition}

	\section{Conservative Planning}
	To solve the safe model-free planning problem, we propose to learn a conservative action model, and then use it to find sound plans.
	
	Following prior work on learning action models~\cite{wang1995learning,wang1994learning,walsh2008efficientLearning}, we partition every observed trajectory $\tuple{s_1,a_1,s_2,\ldots,s_{n+1}}\in\mathcal{T}$ into a set of {\em action triplets}, where each action triplet is of the form $\langle s_i, a_i, s_{i+1}\rangle$. 
	Let $\mathcal{T}(a)$ be all the action triplets for action $a$. 
	A state $s$ and $s'$ are called pre- and post-state of $a$, respectively, if there is an action triplet $\tuple{s,a,s'}$. Following Walsh and Littman~\shortcite{walsh2008efficientLearning} and Wang~\shortcite{wang1994learning,wang1995learning}, we observe that from the set of trajectories $\mathcal{T}$ we can ``bound'' the set of predicates in an action's preconditions and effects, as follows. 
	
	\begin{align}
	\emptyset & \subseteq \pre(a) \subseteq & \bigcap_{\tuple{s, a, s'}\in \mathcal{T}(a)} s \label{eq:pre} \\
	\bigcup_{\tuple{s, a, s'}\in \mathcal{T}(a)} s'\setminus s &
	\subseteq \eff(a) \subseteq &
	\bigcap_{\tuple{s, a, s'}\in \mathcal{T}(a)} s'
	\label{eq:eff}
	\end{align}
	Equation~\ref{eq:pre} holds because a value assignment $(x_i=v)$ cannot be a precondition of $a$ if it is not in every pre-state of $a$, and thus only value assignments that 
	exists in all the pre-states of $a$ may be preconditions of $a$ (hence $\pre(a)\subseteq \bigcap_{\tuple{s, a, s'}\in \mathcal{T}(a)} s$).     On the other hand, the fact that a state variable $x_i$ happened to have the same value $v$ in all the pre-states of the observed trajectories does not necessarily mean that $(x_i=v)$ is a precondition of $a$. It may even be the case that $a$ has no preconditions at all, and thus the ``lower bound'' on an action's precondition is $\emptyset$. 
	
	Equations~\ref{eq:eff} holds because a value assignment $(x_i=v)$ cannot be an effect of $a$ if it is not in every post-state of $a$, and thus only value assignments that 
	exists in all the post-states of $a$ may be preconditions of $a$ (hence $\eff(a)\subseteq \bigcap_{\tuple{s, a, s'}\in \mathcal{T}(a)} s'$). On the other hand, every variable that has a value in the post-state that is different from the pre-state then it must be an effect (hence $\bigcup_{\tuple{s, a, s'}\in \mathcal{T}(a)} s'\setminus s$).
	We denote the ``upper bound'' of the preconditions by $\pre_\mathcal{T}^U(a)$ 
	and the ``lower bound'' of the effects by $\eff_\mathcal{T}^L(a)$.

	\subsection{Compilation to Classical Planning} 
	Next, we use the bounds in Equations~\ref{eq:pre} and~\ref{eq:eff} to compile a \cmfp{} problem 
	$\Pi_\mathcal{T}=\tuple{\mathcal{X},s_\mathcal{I}, s_\mathcal{G}, \mathcal{T}}$ 
	to a classical \SAS{} problem $F(\Pi_\mathcal{T})$, such that a solution to $F(\Pi_\mathcal{T})$ 
	is a solution to $\Pi_\mathcal{T}$, i.e., it is a solution for the underlying planning problem $\Pi$. 
	$F(\Pi_\mathcal{T})$ is defined as follows. It has exactly the same set of 
	state variables ($\mathcal{X}$), start state ($s_\mathcal{I}$), and goal ($s_\mathcal{G}$) as $\Pi_\mathcal{T}$. The actions of $F(\Pi_\mathcal{T})$ is the set of all the actions seen in an observed trajectory. We denote this set of actions by $A(\mathcal{T})$. 
	The preconditions of an action $a$ in  $F(\Pi_\mathcal{T})$ 
	are defined as the ``upper bound'' estimate given in Equation~\ref{eq:pre} ($\pre_\mathcal{T}^U(a)$) and the effects of $a$ in  $F(\Pi_\mathcal{T})$  
	are defined to be the ''lower bound'' estimate given in 
	Equation~\ref{eq:eff} ($\eff_\mathcal{T}^L(a)$).
	
	\begin{definition}[Safe]
		An action model $\mathcal{O}'$ is safe with respect to an action model $\mathcal{O}$
		if for every state $s$ and action $a$ it holds that (1) if $a$ is applicable in $s$ according to $\mathcal{O}'$ then it is also applicable in $s$ according to $\mathcal{O}$, and (2) applying $a$ to $s$ results in exactly the same state when using either $\mathcal{O}'$ or $\mathcal{O}$. 
		\label{def:safe}
	\end{definition}
	
	\begin{theorem}
		The action model in $F(\Pi_\mathcal{T})$ 
		is safe with respect to the action model of $\Pi$. 
		\label{the:safeness}
	\end{theorem}
	\begin{proof}
		Let $\mathcal{O}_\Pi$ and $\mathcal{O}_{F(\Pi_\mathcal{T})}$ be the action models of $\Pi$ and $F(\Pi_\mathcal{T})$, respectively. 
		Since $a$ is applicable in $s$ according to $\mathcal{O}_{F(\Pi_\mathcal{T})}$, it means that $\pre_\mathcal{T}^U(a)\subseteq s$ and consequently  $a$ is also applicable in $s$ according to $\mathcal{O}_\Pi$, since $\pre(a)\subseteq \pre_\mathcal{T}^U$ (Equation~\ref{eq:pre}).

		Now, let $s'_\Pi$ be the state resulting from applying $a$ on $s$ according to $\mathcal{O}_\Pi$, and let $v$ denote the value of a state variable $x_i$ in $s'_\Pi$, i.e., $v=s'_\Pi(x_i)$. 
		Since $\eff_\mathcal{T}^L(a)\subseteq \eff(a)$, 
		then according to $\mathcal{O}_{F(\Pi_\mathcal{T})}$ 
		either $(x_i=v)$ is an effect of $a$ 
		or $a$ has no effect on $x_i$. 
		If the former is true then $s'_{F(\Pi_{\mathcal{T}})}(x_i)=v=s'_\Pi(x_i)$. 
		Otherwise, it means that in the observed trajectories, applying $a$ 
		never changed the value of $x_i$, i.e., $x_i$ was equal to $v$ in both the pre-state and post-state. By definition, this means that $(x_i=v)$ is a precondition of $a$ 
		according to $\mathcal{O}_{F(\Pi_\mathcal{T})}$, 
		and thus $s(x_i)=s'_{F(\Pi_\mathcal{T})}(x_i)=v=s'_\Pi(x_i)$. 
		Thus, the effects of $a$ on $s$ will be the same in both action models, and hence a 
		$s'_\Pi=s'_{F(\Pi_\mathcal{T})}$. 
	\end{proof}
	\begin{corollary}[Soundness]
		Every solution to $F(\Pi_\mathcal{T})$ is also a solution to $\Pi_\mathcal{T}$
		\label{the:soundness}
	\end{corollary}
	Corollary~\ref{the:soundness} is a direct result of Theorem~\ref{the:safeness}, 
	and its practical implication is the following algorithm for solving 
	any \cmfp{} problem $\Pi_\mathcal{T}$: compile it to a classical planning problem 
	$F(\Pi_\mathcal{T})$, run an off-the-shelf classical planner, and return the resulting plan. 
	We refer to this algorithm as the {\em conservative model-free planner}. 	
	The {\em conservative model-free planner} is sound, but it is not complete. 
	There can be planning problems 
	that have a solution but the observed trajectories are not sufficient 
	to induce a corresponding compiled planning problem that is solvable. 
	As an extreme example, if we do not receive any observed trajectories, 
	the compiled planning problem will not have any actions in its action model
	and thus will not be able to solve anything. 
	In the next section we show that the required number of trajectories is actually reasonable.

	\begin{figure*}
		\includegraphics[width=\textwidth]{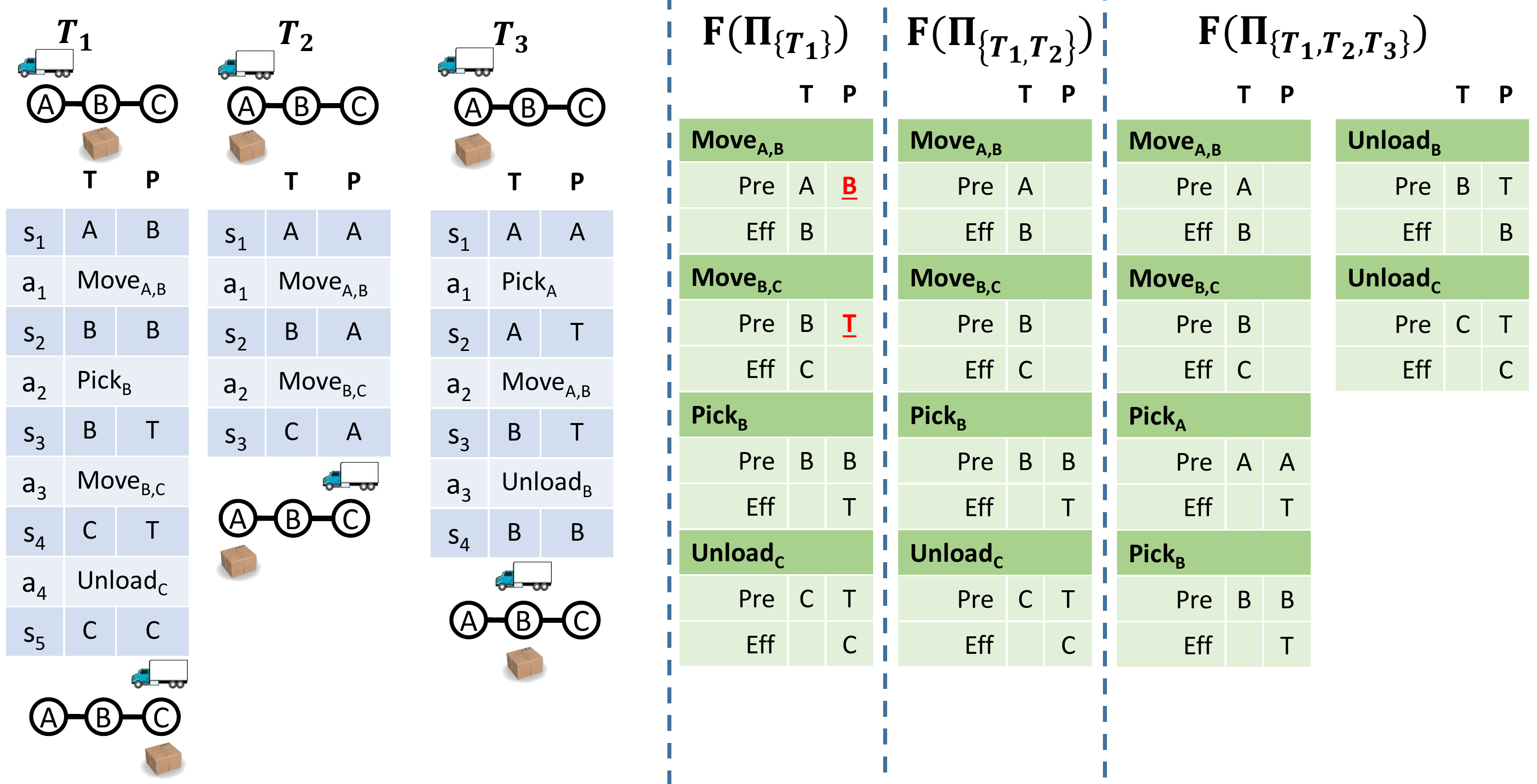}
	\caption{An example of creating $F(\Pi_{\mathcal{T}})$ from observed trajectories in a simple logistics-like domain. The left-hand side of the figure shows the three given trajectories $T_1$, $T_2$, and $T_3$. The right hand side of the figure shows our conservative action model, 
		learned using different subsets of the trajectories $T_1$, $T_2$, and $T_3$ together.}
	\label{fig:example}
	\end{figure*}

	Figure~\ref{fig:example} illustrates how to generate  $F(\Pi_{\mathcal{T}})$ from observed trajectories in a simple logistic-like domain with one truck, one package, and three possible location $A$, $B$, and $C$. The state variables are TruckAt, with domain $A$, $B$, and $C$, and PackageAt with domain $A$, $B$, $C$, and $T$, where $T$ represents that the package is on the truck. The possible actions are Move$_{X,Y}$, 
	Pickup$_X$, and Unload$_X$, for every $X,Y\in \{A,B,C\}$.

	The three tables on the left-hand side, $T_1$, $T_2$, and $T_3$ are three observed trajectories, where the column $T$ represents the value of the state variable TruckAt and the column $P$ represents the value of the state variable PackageAt. For example, $T_1$ represents a trajectory where the truck starts at $A$, moves to $B$, picks up the package, moves to $C$, and unloads the package there. 	The tables on the right-hand side of Figure~\ref{fig:example} show the action model learned from these trajectories. For didactic reasons, we show the action model learned given just $T_1$ ($F(\Pi_{\{T_1\}})$), 
	then the action model learned given $T_1$ and $T_2$ ($F(\Pi_{\{T_1, T_2\}})$), and finally 
	the action model learned using all three trajectories.

	As can be observed, given only $T_1$ we do not have any knowledge of many actions such as Pick$_A$, Pick$_C$, etc. Also, the preconditions learned for the actions Move$_{A,B}$ 
	and Move$_{B,C}$ are too restrictive, requiring that the package is at some location (while clearly a Move action only requires knowing the truck's location). However, given $T_1$ and $T_2$, these redundant preconditions are removed, and thus task that can be achieved with the actions Move$_{A,B}$, Move$_{B,C}$, Pick$_B$, and Unload$_C$ will be found by our conservative model-free planner.

	\section{Learning to Perform Safe Planning}\label{learning-sec}
	In general, we cannot guarantee that {\em any finite number of trajectories} will suffice to obtain {\em precisely} the underlying action model. This is because, for example, if some action never appears in a trajectory, we may not know its effects; or, if an action is only used in a limited variety of states, it may be impossible to distinguish its preconditions. Consequently, we cannot guarantee a complete solution to the model-free safe planning problem. However, as the number of trajectories increases, we can hope to learn enough of the actions accurately enough to be able to find plans for most goals in practice. This gives raise to a statistical view of the model-free safe planning task (Definition~\ref{def:model-free-planning}) that follows the usual statistical view of learning, along the lines of Vapnik and Chervonenkis~\shortcite{vapnik1971} and Valiant~\shortcite{valiant1984}.  
	
	\begin{definition}[Safe Model-Free Learning-to-Plan] We suppose that there is an arbitrary, unknown (prior) probability distribution $D$ over triples of the form $\tuple{s_\mathcal{I}, s_\mathcal{G}, T}$, where $s_\mathcal{I}$ is a state, $s_\mathcal{G}$ is a goal condition, and $T$ is a trajectory that starts in $s_\mathcal{I}$ and ends in a state that satisfies $s_\mathcal{G}$, 
		and all trajectories are applicable in a fixed planning domain $\mathcal{D}=\tuple{\mathcal{X}, \mathcal{O}}$. 
		In the safe model-free learning-to-plan task, we are given a set of triplets $\tuple{s_\mathcal{I}^{(1)},s_\mathcal{G}^{(1)},T^{(1)}},\ldots,\tuple{s_\mathcal{I}^{(m)},s_\mathcal{G}^{(m)},T^{(m)}}$ drawn independently from $D$, 
		and a new \SAS{} planning problem $\Pi=\tuple{\mathcal{D}, s_\mathcal{I}, s_\mathcal{G}}$  
		such that the initial state and goal condition $(s_\mathcal{I},G)$ are from some $(s_\mathcal{I},s_\mathcal{G},T)$ drawn from $D$. 
		The task is to either output a plan for $\Pi$ or, with probability at most $\epsilon$, return that no plan was found.
		\label{def:learning-to-plan}
	\end{definition}

	\subsubsection{Remarks on the task formulation}
	We stress that $D$ is arbitrary, and thus  the conditional distribution
	of trajectories given a start and goal state, $D(T|s_{\mathcal{I}},s_{\mathcal{G}})$, can also be any arbitrary distribution. For example, $D$ could be the distribution of trajectories obtained
	by running some (unknown, sophisticated) planning algorithm on input $(s_{\mathcal{I}},s_{\mathcal{G}})$, or produced by hand by a human domain expert. 
	More generally, the adversarial choice of $D$ in our model subsumes a model in which a trajectory $T(s_{\mathcal{I}},s_{\mathcal{G}})$ is nondeterministically and adversarially chosen to
	be paired with the start state $s_{\mathcal{I}}$ and goal $s_{\mathcal{G}}$. 
	(Indeed, the distributions used in such a case satisfy the stronger restriction that $D(T|s_{\mathcal{I}},s_{\mathcal{G}})$ produces a deterministic outcome $T(s_{\mathcal{I}},s_{\mathcal{G}})$, which does not necessarily hold in our model.)
	
	We also note that our conservative model-free planner does not actually require knowledge of the goals $s_{\mathcal{G}}$ associated with the trajectories drawn from $D$. Thus, our approach actually solves a more demanding task that does not provide the goals to the learning algorithm. But, such a distribution over goals is nevertheless a central feature in our notion of ``approximate completeness,'' and features prominently in the analysis as we discuss next.
	
	\subsubsection{Analysis of learning}
	A key question is how our certainty that a plan can be generated for a new start and goal state grows with the number of trajectories. Let $\log$ and $\ln$ denote $\log_2$ and $\log_e$, respectively, and let $d$ denote the largest number of values
	for a state variable. 
	\begin{theorem}
		Using the conservative model-free planner, 
		it is sufficient to observe 
		$m\geq\frac{(2\ln  d)|A|}{\epsilon}(|\mathcal{X}|+\log\frac{2|A|}{\delta})$
		trajectories to solve the safe model-free learning-to-plan problem with probability $1-\delta$. 
		\label{the:pac-conformant}
	\end{theorem}
	{\bf Proof Outline.} First, Lemma~\ref{lem:sufficientActions} shows 
		that the set of actions used by our conservative model-free planner ($A(\mathcal{T})$) 
		is sufficient to solve a randomly drawn problem with high probability. 
		Then, Lemma~\ref{lem:plan-existance} shows that 
		under certain conditions, the preconditions we learned for these actions ($\pre_\mathcal{L}^U$) 
		are not too conservative, i.e., they are adequate
		for finding a plan with high probability.  
		Finally, we prove that with high probability these conditions over the action model we learned 
		indeed hold.

		
		\begin{lemma}
			Let $A_\epsilon$ be the set of actions such that each action $a\in A_\epsilon$ appears in a trajectory sampled from $D$ with probability at least $\frac{\epsilon}{2|A|}$ and let $A(\mathcal{T})$ be the set of every action that appeared in a trace. 
			The probability that all the actions $a\in A_\epsilon$ 
			appear in $A(\mathcal{T})$ is at least $1-\delta/2$.
			\label{lem:sufficientActions}
		\end{lemma}
		\begin{proof}
			By definition, the probability that an action $a\in A_\epsilon$ does not exist in a trajectory drawn from $D$ is $1-\frac{\epsilon}{2|A|}$. 
			Since the observed trajectories $\mathcal{T}$ are drawn independently from $D$ we have that the probability that $a\notin A(\mathcal{T})$ is
			$(1-\frac{\epsilon}{2|A|})^m\leq e^{-\frac{\epsilon\cdot m}{2|A|}}$, 
			using the inequality $1-x\leq e^{-x}$. 
			Since we assume in Theorem~\ref{the:pac-conformant} that $m\geq\frac{(2\ln d)|A|}{\epsilon}(|\mathcal{X}|+\log\frac{2|A|}{\delta})$
			which is larger than $\frac{(2\ln 2)|A|}{\epsilon}\log\frac{2|A|}{\delta}$, 
			we have that the probability that $a\notin A(\mathcal{T})$ is at most 
			\begin{equation}
			e^{-\frac{\epsilon\cdot m}{2|A|}}
			\leq e^{-\frac{\epsilon}{2|A|} \cdot \frac{(2\ln 2)|A|}{\epsilon}\log\frac{2|A|}{\delta}}
			=e^{-\ln\frac{2|A|}{\delta}}=\frac{\delta}{2|A|}
			\end{equation}
			Hence, by a union bound over $a\in A_\epsilon$ (noting $|A_\epsilon|\leq |A|$), we have that $A_\epsilon\subseteq A(\mathcal{T})$ with probability $1-\delta/2$ as needed.
		\end{proof}
		
		Stated informally, Lemma~\ref{lem:sufficientActions} says that with high probability we will observe
		all the ``useful'' actions, i.e., the actions used in many trajectories. However, we may have learned 
		preconditions for these actions that are too conservative, 
		preventing the planner from finding a plan even if one exists. 
		We next define a property of action models that states that this does not occur frequently. 
		
		
		\begin{definition}[Adequate]
			We call an action model {\em $\epsilon$-adequate} if, with probability at most $\epsilon/2$, we sample a trajectory $T$ from $D$ such that $T$ contains an action triplet $\tuple{s,a,s'}$  where $a\in A_\epsilon$ and $s$ does not satisfy $\pre_\mathcal{T}^U$. 
			\label{def:adequate}
		\end{definition}
		We say that an action model is {\em $\epsilon$-inadequate} if it is not $\epsilon$-adequate. 
		An equivalent way to define the notion of an {\em $\epsilon$-adequate} action model is
		that with probability at most $\epsilon/2$ 
		a trajectory is sampled in which an action $a\in A_\epsilon$ is invoked on a state that 
		does not satisfy the conservative preconditions of $a$ we learned from the given set of trajectories ($\mathcal{T}$).

		\begin{lemma}
			If the learned action model is $\epsilon$-adequate and $A_\epsilon\subseteq A(\mathcal{T})$, 
			then with probability $1-\epsilon$ our conservative model-free planner will find a plan for a start-goal pair $(s_\mathcal{I},s_\mathcal{G})$ sampled from $D$.
			\label{lem:plan-existance}
		\end{lemma}
		\begin{proof}
			Let $T$ be the (unknown) trajectory sampled for $(s_\mathcal{I},s_\mathcal{G})$. 
			The probability that $T$ uses an action that is not in 
			$A_\epsilon$ is at most $|A|\cdot\frac{\epsilon}{2|A|}=\epsilon/2$. 
			Thus, $T$ contains only actions known to our planner with probability at least $1-\epsilon/2$, since we assumed that $A_\epsilon\subseteq A(\mathcal{T})$. 
			Since the action model is $\epsilon$-adequate then with probability $1-\epsilon/2$ the learned preconditions are satisfied on all of the states in $T$. Thus, by a union bound, we find that with probability $1-\epsilon$, our planner could at least find $\mathcal{T}$. Hence it will find a trajectory from $s_\mathcal{I}$ to $s_\mathcal{G}$, as required. 
		\end{proof}
		
		\begin{lemma}
            The action model used by our conservative model-free planner is $\epsilon$-adequate with probability at least $1-\delta/2$.
			\label{lem:action-model-ok}
		\end{lemma}
		\begin{proof}
		Whether an action model is $\epsilon$-adequate or not depends on the assignment of preconditions to actions. 
		Since there are $|\mathcal{X}|$ state variables each with at most $d$ values, 
		then there are $d^{|\mathcal{X}|}$ possible assignments of preconditions for an individual action
		and a total of $d^{|\mathcal{X}|\cdot |A|}$ possible preconditions assignments for an action model . 
		Let $BAD$ be the subset of these action model preconditions assignments that are not $\epsilon$-adequate. 
		Clearly, $BAD$ has size at most $d^{|\mathcal{X}|\cdot |A|}$.

		Consider a particular assignment of preconditions in an $\epsilon$-inadequate action model $\mathcal{O}_{BAD}$. Since $\mathcal{O}_{BAD}$ is $\epsilon$-inadequate, 
		it has a set of state-action pairs ($s,a$) associated with it such that $a\in A_\epsilon$ and $a$ cannot be applied to $s$ according to $\mathcal{O}_{BAD}$. 
		The action model $\mathcal{O}_{BAD}$ can only be learned by our algorithm 
		if none of these state-action pairs were observed in the given trajectories $\mathcal{T}$. 
		On the other hand, by the definition of inadequacy
		the probability of having a state-action pair from that list 
		in a trajectory drawn from $D$ is at least $\epsilon/2$. 
		Thus, the probability that our algorithm will learn a particular preconditions assignment of an $\epsilon$-inadequate action model is at most $(1-\epsilon/2)^m$.
		Since $m\geq\frac{(2\ln  d)|A|}{\epsilon}(|\mathcal{X}|+\log\frac{2|A|}{\delta})>\frac{2\ln d}{\epsilon}(|A|\cdot |\mathcal{X}|+\log\frac{2}{\delta})$, then $(1-\epsilon/2)^m$ is smaller than 	
		\begin{equation}
		e^{-\frac{\epsilon}{2} \cdot m}
		\leq 
		e^{-(\ln d \cdot (|A|\cdot |\mathcal{X}|+\log\frac{2}{\delta}))}
		\leq e^{-\ln d \cdot |A|\cdot |\mathcal{X}|}\cdot e^{-\ln  \frac{2}{\delta}}
		\end{equation}
		which is at most $\frac{\delta}{2\cdot d^{|\mathcal{X}|\cdot |A|}}$. Thus, by a union bound over this set of inadequate assignments $BAD$, the probability that any inadequate assignment of preconditions could be output is at most $\delta/2$.
		Thus, with probability $1-\delta$, the algorithm indeed produces an assignment of preconditions that is neither unsafe for $A(\mathcal{T})$ nor inadequate for $A_\epsilon$, as needed.
	\end{proof}
	
	\subsection{Unsolvable Instances}
	
	The implication of Theorem~\ref{the:pac-conformant} is that by observing 
	a number of trajectories that is quasi-linear in the number of actions and the number of state variables, 
	we expect our safe model-free planner to be complete with high probability, in the sense that if a solution exists it will be found. But what if some of the drawn problem instances are not solvable? 

\begin{table}
	\centering
	\begin{tabular}{c|c|c|c}
			& $\plannable (\Pi)$ & $\neg\plannable(\Pi)$ & Priors \\ \hline
		$\solvable(\Pi)$ 		& $1-\epsilon$ & $\epsilon$ & $\mu$ \\
		$\neg\solvable(\Pi)$	& 0 & 1 & $1-\mu$ \\ \hline
		Marginals	& $\mu (1-\epsilon)$ & $1-\mu(1-\epsilon)$ &
	\end{tabular}
\caption{The table shows the priors, conditional probabilities, and marginals obtained with probability $1-\delta$ when using our conservative model-free planner when it is given $m\geq\frac{(2\ln  d)|A|}{\epsilon}(|\mathcal{X}|+\log\frac{2|A|}{\delta})$ trajectories.}
\label{tab:probabilities}	
\end{table}

	\begin{corollary}
		If the probability of drawing a solvable start-goal pair from $D$ is $\mu$ then it is sufficient to 
		observe $m\geq\frac{(2\ln  d)|A|}{\epsilon}(|\mathcal{X}|+\log\frac{2|A|}{\delta})$
		trajectories (of solvable instances) to guarantee that with probability of at least $1-\delta$ our conservative model-free planner will solve a start-goal pair drawn from $D$ with probability at least $\mu\cdot(1-\epsilon)$ 
\label{cor:unsolvable}
	\end{corollary}
\begin{proof}
	Let $\solvable(\Pi)$ be true or false if  a given planning problem $\Pi$ is solvable or unsolvable, respectively, 
	and let $\plannable(\Pi)$ be true or false if our planner returns a solution to $\Pi$ or not respectively. We aim to bound $Pr(\plannable(\Pi))$. Since our planner is sound, $\plannable(\Pi)\rightarrow \solvable(\Pi)$ and so
	\[ Pr(\plannable(\Pi))= Pr(\plannable(\Pi)\wedge \solvable(\Pi))= Pr(\plannable(\Pi)|\solvable(\Pi))\cdot Pr(\solvable(\Pi)) \]
	According to Theorem~\ref{the:pac-conformant}, 
	$Pr(\plannable(\Pi)|\solvable(\Pi))\geq 1-\epsilon$ 
	and $Pr(\solvable(\Pi))=\mu$ by definition. 
\end{proof}

Table~\ref{tab:probabilities} shows the priors, conditional probabilities, and marginals use by the proof of Corollary~\ref{cor:unsolvable}
The first row shows the probabilities 
	$Pr(\plannable(\Pi)|\solvable(\Pi))$, 
	$Pr(\neg\plannable(\Pi)|\solvable(\Pi))$, abd
	$Pr(\solvable(\Pi))$; 
	the second row shows the probabilities 
	$Pr(\plannable(\Pi)|\neg\solvable(\Pi))$, 
	$Pr(\neg\plannable(\Pi)|\neg\solvable(\Pi))$, and
	$Pr(\neg\solvable(\Pi))$;  	
	and the last row shows the marginal probabilities
	$Pr(\plannable(\Pi))$ and  
	$Pr(\neg\plannable(\Pi))$. 
	
	Corollary~\ref{cor:unsolvable} and Table~\ref{tab:probabilities} are valuable in that they provides a relationship between $\mu$, $\epsilon$, $\delta$, and $m$. Thus, we can increase $m$ to satisfy more demanding values of $\mu$, $\epsilon$, and $\delta$ and different types of error bounds. For example, consider an application that requires bounding, by some $\gamma$, the probability that our planner outputs incorrectly that no plan exists. In other words, an application that requires 
	\[ Pr(\solvable(\Pi)|\neg\plannable(\Pi))\leq \gamma \]
	Using Bayes' rule and the values from Table~\ref{tab:probabilities}, this means that 
	\[ \frac{\epsilon\cdot\mu}{1-(1-\epsilon)\cdot\mu}\leq \gamma \Leftrightarrow  \epsilon\leq \frac{\gamma\cdot (1-\mu)}{\mu\cdot (1+\gamma)}\]
	\noindent Plugging $\frac{\gamma\cdot (1-\mu)}{\mu\cdot (1+\gamma)}$ into the sample complexity instead of $\epsilon$ in Theorem~\ref{the:pac-conformant} will give the required number of trajectories to obtain a bound of  $\gamma$ on the probability of incorrectly outputting that a problem is not solvable. 
	
	\subsection{Limited Planner Capabilities}
	The given trajectories $\mathcal{T}$ are presumably generated by some planning entity. Since planning in general is a hard problem, it may be the case that the planner that generated the given set of trajectories has drawn a solvable problem from $D$ but was just not able to solve it due to memory or time constraints. 
	
	Learning from such a set of trajectories does not enable bounding the probability of solving problems in general. What can be obtained in such cases is to bound the solving capabilities of our conservative model-free planner with respect to the capabilities of the planner that generated the observed trajectories. Thus, instead of having $\mu$ represent the probability that an instance is solvable, we will have $\mu$ represent the probability that an instance is solvable by the original planner. The rest of the analysis follows exactly the same as in the previous section.

	\section{Related Work}\label{relatedwork-sec}
	
	Our work relates to several well-studied types of problems: planning under uncertainty, reinforcement learning, and domain model learning. 
	\paragraph{Planning under uncertainty.} In common models for planning under uncertainty, 
	such as Markov Decision Problems (MDP) and Partially Observable MDPs (POMDP), the uncertainty stems from the stochastic nature of the world or from imperfect sensors that prevent full observability of the agent's state. Our setting is different in that our uncertainty only stems from not knowing the agent's action model. 
	
	\paragraph{Reinforcement learning.} Reinforcement learning algorithms learn how to act by interacting with the environment. Thus, they are designed for a trial-and-error approach to learn the domain and/or how to plan in it. Our task is to generate a plan that {\em must} work, so a trial-and-error approach is not sufficient. 
	
	\paragraph{Domain model learning.}
	 Most prior work on learning a domain model in general or a STRIPS action model from observed trajectories, such as ARMS~\cite{yang2007learning} and LOCM~\cite{cresswell2013acquiring}, learn approximate models that do not guarantee safety. Hence, such work generally also involves some form of trial-and-error as well, iteratively requesting more example trajectories or interacting directly with the environment to refine the learned model~\cite{mourao2012learning,wang1994learning,wang1994learning,walsh2008efficientLearning,levine2006explanation,jimenez2013integrating}.
	In addition, most works learn from both positive and negative examples -- observing successful and failed trajectories, while we only require successful trajectories to be provided.

	Another key difference is that unlike our work, most prior works do not provide statistical guarantees on the soundness of the plan generated with their learned model.
	An exception to this is the work of Walsh and Littman~\shortcite{walsh2008efficientLearning}, that also discussed the problem of learning STRIPS operators from observed trajectories and provided theoretical bounds on the sample complexity -- the number of interactions that may fail until the resulting planner is  sound and complete.  By contrast, we do not assume any planning and execution loop and do not allow failed interactions. Hence, we aim for a planning algorithm that is guaranteed to be sound, at the cost of completeness. This difference affects their approach to learning. They attempted to follow an optimistic assumption about the preconditions and effects of the learned actions, in an effort to identify inaccuracies in their action model. By contrast, we are forced to take a pessimistic approach, as we aim for a successful execution of the plan rather than information gathering to improve the action model.

	\section{Conclusions}
	This paper deals with a planning problem in which the planner agent has no knowledge about its actions. Instead of an action model, the planner is given a set of observed trajectories of successfully executed plans. 	In this setting we introduced the {\em safe model-free planning} problem, in which the task is to find a plan that is guaranteed to reach the goal, i.e., there is no tolerance for execution failure. This type of problem is important in cases where failure is costly or in cases where the agent has no capability to replan during execution.

	We showed how to use the given set of trajectories to learn
	about the agent's actions, bounding the set of predicates in the actions'  preconditions and effects. Then, we proposed a conservative approach to solve the safe model-free problem that is based on a translation to a classical planning problem. This solution is sound but is not complete, as it may fail to find a solution even if one exists. However, we prove that under some assumptions the likelihood of finding a solution with this approach grows linearly with the number of predicates and quasi-linearly with the number of actions.

	Future directions for safe model-free planning include studying how to address richer underlying planning models including parametrized actions, conditional effects, stochastic action outcomes, and partial observability. While some of these more complex action models can be compiled away (e.g., a problem with conditional effects can be compiled to a problem without conditional effects~\cite{nebel2000compilability}), the resulting problem can be significantly larger. A particularly interesting direction is how to learn there lifted action model, i.e., 
	what can be learned from a trajectory 
	with an action $a(x)$ on the action 
	model of $a(y)$, where $a$ is a parameterized action and $x$ and $y$ are different values for the same parameter.

	
	\section*{Acknowledgements}
	B.~Juba was partially supported by an AFOSR Young Investigator Award. 
	R.~Stern was partially supported by the Cyber Security Research Center 
	at BGU. 
	
	\bibliographystyle{named}
	\bibliography{library}
	
\end{document}